\documentclass{article}

\PassOptionsToPackage{numbers, compress}{natbib}

\usepackage[preprint]{neurips_2020}

\usepackage[utf8]{inputenc} 
\usepackage[T1]{fontenc}    
\usepackage{hyperref}       
\usepackage{url}            
\usepackage{booktabs}       
\usepackage{amsfonts}       
\usepackage{nicefrac}       
\usepackage{microtype}      
\usepackage[pdftex]{graphicx}
\usepackage[linesnumbered,ruled]{algorithm2e}
\usepackage{array}
\usepackage{dsfont}
\usepackage{amsmath}
\usepackage{amsthm}
\usepackage{multirow}
\usepackage{multicol}

\usepackage[dvipsnames]{xcolor}

\newtheorem{theorem}{Theorem}
\newtheorem{lemma}{Lemma}
\newtheorem{proposition}{Proposition}
\newtheorem{property}{Property}

\newtheorem{corollary}{Corollary}

\theoremstyle{definition}
\newtheorem{definition}{Definition}
\newtheorem{notation}{Notation}
\newcommand{\Chi}{\mathrm{M}}

\usepackage{natbib}

\title{X-SHAP: towards multiplicative explainability of Machine Learning}

\author{%
  Luisa Bouneder\\
  Emerton Data\\
  \texttt{luisa.bouneder@emerton-data.com} \\
  \And
  Yannick Léo\\
  Emerton Data\\
  \texttt{yannick.leo@emerton-data.com} \\
  \And
  Aimé Lachapelle\thanks{www.emerton-data.com} \\
  Emerton Data\\
  \texttt{aime.lachapelle@emerton-data.com}
}

\begin{document}

\maketitle

\begin{abstract}

This paper introduces X-SHAP, a model-agnostic method that assesses multiplicative contributions of variables for both local and global predictions. This method theoretically and operationally extends the so-called additive SHAP approach. It proves useful underlying multiplicative interactions of factors, typically arising in sectors where Generalized Linear Models are traditionally used, such as in insurance or biology. We test the method on various datasets and propose a set of techniques based on individual X-SHAP contributions to build aggregated multiplicative contributions and to capture multiplicative feature importance, that we compare to traditional techniques. 
\end{abstract}

\section{Introduction}

Interpretation of prediction model outputs can be as important as the prediction of machine learning models, e.g. insurance pricing, credit rejection or acceptance, recommendation to decision markers,  medical diagnostic. The users need to understand the factors underlying the prediction. Model interpretability offers the possibility to better audit the robustness and fairness of predictive models. Simple models such as linear regressions or GLMs are quite accurate and easily interpretable. On the contrary, the development of more complex models, such as machine learning ensemble models or deep learning models leads, to highly accurate but more complex models that are difficult to interpret. The trade-off between building a more accurate model vs. keeping a simple and interpretable model is not an easy choice. In many cases, the simple interpretable model is still preferred. In order to solve the accuracy-interpretability trade-off, a large number of interpretable methods have been proposed~\citep{ribeiro2018lime, lundberg2017shapgithub, shrikumar2017learning, strumbelj2013shap, lundberg2020local, bach2015pixel, datta2016algorithmic}. It is noteworthy that all these methods focus on additive contributions computation, none of them being able to tackle multiplicative contributions assessment.

In this paper, we introduce, X-SHAP, a model-agnostic interpretability method that provides multiplicative contributions for individual predictions. Our main contributions are summarized as follows:
\begin{enumerate}
    \item We extend the additive analytical solution to the model-agnostic multiplicative interpretability problem,
    \item We introduce X-SHAP, an algorithm that provides approximate multiplicative contributions at individual levels,
    \item We propose the X-SHAP toolbox, a new set of techniques used to understand global and segmented model structure by aggregating multiple local contributions,
    \item We empirically verify desirable properties and compare the X-SHAP approach to both the additive algorithm Kernel SHAP, and to well-known metrics on various supervised problems.
\end{enumerate}

\section{Related work}

The simplest way to interpret any prediction model's outputs is to analyze the model itself when it is not too complex. This is the case for simple models like Generalized Linear Models~\citep{nelder1972generalized, antonio2007actuarial, mcneil2007bayesian} or decision trees~\citep{safavian1991survey}, yet, more complex models are not directly interpretable.

To raise adoption of complex models, specific interpretable methods have been developed. Although neural networks have a black box nature, some interpretable approaches exist~\citep{bach2015pixel,shrikumar2017learning}. For instance, DeepLIFT~\citep{shrikumar2017learning} (Deep Learning Important FeaTures) decomposes the output prediction of a neural network on a specific input by backpropagating the contributions of all neurons in the network to each feature of the input. In order to interpret tree based machine learning ensemble models such as random forests or gradient boosting,~\citet{lundberg2020local} proposes a polynomial time explainer based on game theory that measures local feature interaction effects.

There are two types of model-agnostic interpretability methods. The first type consists of finding the training points that are most responsible for the prediction ~\citep{influence_functions, datta2016algorithmic}. The second type of general explainer performs a local linear regression around the prediction and extracts contributions from local linear models~\citep{ribeiro2018lime}. In this case, when features are not independent, contributions are produced via Shapley values, a concept in cooperative game theory, introduced in ~\citep{shapley1953value}, that assigns a unique distribution (among the players) of a total surplus generated by the coalition of all players. These are the SHAP methods~\citep{strumbelj2013shap, lundberg2017shapgithub}.

In many fields such as actuarial~\citep{antonio2007actuarial, goldburg2016glm}, epidemiology~\citep{land2000multiplicative}, economy~\citep{wang2017multiplicative} and medicine~\citep{mehta2016major} phenomenon are multiplicative by nature, very often with a traditional use of models (e.g. log-GLM), the available interpretability methods provide additive interpretations. Little attention has been paid to multiplicative contributions assessment despite the existence of theoretical extension of additive Shapley values~\citep{young1985monotonic} to multiplicative provided by~\citet{ortmann2013multiplicativeshap} to positive cooperative games. In this paper, we propose to fill this gap by extending the Kernel SHAP local interpretation method to multiplicative problems.

\section{Problem and notations}

\subsection{Model-agnostic interpretability problem}\label{sec:problem}

Let $X$ be an input dataset composed of $n$ observations $x_i$ and $m$ features where $X=\{x{_i}^j\}$ with $\forall i \in [1, n], \forall j \in [1, m], x{_i}^j \in \mathbb{R}$. $x_i$ refers to a single observation of the dataset $X$. The set of features $\{j\}_{j\in [1,m]}$ is noted $F$. Let us introduce a strictly positive target vector $Y=\{y_i\}_{i \in[1,n]}$ such that $\forall i \in [1, n], y_i > 0$. Let $f$ denotes the associated predictive model $f: \mathbb{R}^{k} \rightarrow \mathbb{R}^{+*}, \forall i \in [1, n], \hat{y}_i = f(x_i)$. Let us assume that the predictive model $f$ is already trained on the dataset $(X_{train}, Y_{train})$ with same properties as $(X, Y)$.

The usual method used to explain machine learning models is the additive contributions of features.

\begin{definition} \textbf{Additive feature contributions.} Let $f$ be a predictive model associated with $(X,Y)$ and $x_i$ a single observation of $X$ with $\hat{y}_i=f(x_i)$. The prediction of $x_i$ can be decomposed by the sum of the additive feature contributions:
\begin{align}\label{eq:shap_def}
    \phi^{0} + \sum \limits_{j = 1}^{m} \phi_{i}^{j}(x_i) = f(x_i) = \hat{y}_{i}
\end{align}

where $\phi^{0}$ is a baseline value for predictions, independent of the observations explained, $m$ is the number of features, $\phi_{i}^{j}$ is the additive contribution of feature $j$ to the model prediction $\hat{y}_{i}$ for the observation $x_{i}$. $\phi$ or $\phi_f$ denotes the set of additive contributions related to $f$.
\end{definition}

In this paper, we focus on use the multiplicative contributions of features.

\begin{definition} \textbf{Multiplicative feature contributions.} Let $f$ be a predictive model associated with $(X,Y)$ and $x_i$ a single observation of $X$ with $\hat{y}_i=f(x_i)$. The prediction of a single observation $x_i$, also refers to $x$ to simplify, can be decomposed by the product multiplicative feature contributions:
\begin{align}\label{eq:xshap_def}
    \psi^{0} \times \prod \limits_{j = 1}^{m} \psi_{i}^{j}(x_i) = f(x_i) = \hat{y}_{i}
\end{align}

where $\psi^{0}$ is a baseline value for predictions, independent of the observations explained, $m$ is the number of features, $\psi_{i}^{j}$ is the multiplicative contribution of feature $j$ to the model prediction $\hat{y}_{i}$ for the observation $x_{i}$. We note $\psi$ or $\psi_f$, the set of multiplicative contributions related to $f$.
\end{definition}

\textbf{Model-agnostic interpretability problem} feeds as follows: given any predictive model $f$ associated with the dataset $(X, Y)$, the multiplicative (resp. additive) model-agnostic interpretability problem consists of finding, for any prediction $(x_i, \hat{y}_i)$, a multiplicative (resp. additive) feature contributions $\psi$ (resp. $\phi$).

\subsection{Notation and definitions}

\begin{notation} \textbf{Arithmetic and geometric means.}
Considering $n$ real values $\forall i \in [1,n], x_i \in \mathbb{R}$, the arithmetic mean is noted ${<x>}_{+} = \frac{1}{n} \sum_{i=1}^{n}{x_i}$ and the geometric mean is noted ${<x>}_{\times} = {(\prod_{i=1}^{n}{x_i})}^{\frac{1}{n}}$
\end{notation}

\begin{definition} \textbf{Coalition vector.}
We define the coalition vector $c$ of dimension $m$ as a simple binary vector $c \in [0, 1]^m$ representing a set of activated features with $c$ of $F$. The complementary coalition vector, noted $\bar{c}$, is defined as follows: $\forall j \in [1, m], \bar{c}^j=1-c^j$. $c$ can also be noted $c_k$ when multiple coalitions have to be enumerated.
\end{definition}

\begin{definition} \textbf{Sub-observation and sub-dataset.}
Considering an observation $x_i$ of a dataset $X$ and a coalition vector $c \subset F$, the induced  sub-observation is given by ${x_i}^c = x_i \times c$. One can extend to the sub dataset $X^c=X*c$.
\end{definition}

\begin{definition} \textbf{Augmented observation.}
Considering an observation $x_i$ of size $m$ of a dataset $X$, the augmented dataset is defined as the duplicate ($n$ times) of $x_i$: $X_i = x_i \times \mathds{1}_n$. Thus, the size of the matrix $X_i$ is $n\times m$
\end{definition}

\begin{definition} \textbf{Perturbated coalition dataset.}
Considering an augmented observation $X_i$ of an observation $x_i \in X$ and a coalition vector $c$ of $F$, we define the perturbated coalition as $\Chi^c(X, x_i) = X_i^c+X^{\bar{c}}$
\end{definition}

\section{Short review of the Kernel SHAP method used for additive contributions}

Before introducing the X-SHAP method, end for the sake of comparison and clarity, we remind the Kernel SHAP method from which it is derived.

A theoretical solution to the additive version of the model-agnostic interpretation problem is introduced in ~\citep{young1985monotonic,strumbelj2013shap,lundberg2017shapgithub}. It shows that additive Shapley values defined in eq. \eqref{eq:add_shapley_values} are the unique solution to the additive model-agnostic interpretability problem defined in section~\ref{sec:problem} that respects local accuracy, missingness and consistency properties defined in ~\citep{lundberg2017shapgithub}. The solution is given by:

\begin{align}\label{eq:add_shapley_values}
    \phi^{j}(x) = \sum \limits_{c \subset F \setminus \{j\}} \frac{|c|!(|F| - |c| - 1)!}{|F|!}(f_{c \cup  \{j\}}(x_{c \cup \{j\}})-f_{c}(x_{c})) 
\end{align}
where $x$ is the considered observation, $f$ the predictive model and $f_{c}(x_{c})$ is the prediction of the model restrained to the space of features $c$ applied to sub-observation $x_{c}$, $F$ is the set of features. For all coalitions, combinatory fractions are noted as the weights $W$.

In practice, for a given dataset $X$, the additive contribution of a feature $j$ is averaged among multiple observations. It can be proven that the problem of computing the Shapley value is an NP-complete problem. Therefore, ~\citet{lundberg2017shapgithub} propose the Kernel SHAP method to approximate the additive feature contributions $\tilde{\phi}\approx\phi$. A python library of the Kernel SHAP algorihtm is implemented and available\footnote{\url{https://github.com/slundberg/shap}}. To do so, ~\citep{lundberg2017shapgithub} makes two main simplifications:
\begin{enumerate}
    \item First, in order to obtain linear computation of the Shapley values, as proposed in~\citet{castro2009polynomial}, not all the coalitions are enumerated. The selection of coalitions is done in order of importance in the Shapley values formula (eq.\eqref{eq:add_shapley_values}) measured by the weights $W$. First come coalitions of size $1$ (all singletons) and their respective complementary (of size $m - 1$), then all coalitions of size $2$ paired with their complementary (of size $m - 2$), and so on
    \item Second, a representative sample $X^{ref}$ of the whole dataset $X$ containing $n^{ref} << n$ observations is considered to compute contributions. Thus, the average reference target value is $\hat{y}^{ref} = <f(X^{ref}>_+$
\end{enumerate}

Then, in order to compute the additive contributions of an observation $x_i$, the perturbated coalition dataset $\Chi^c(X, x_i)$ is built for each coalition $c\in C$ as follows: $\forall c \in C, \Chi^c(X^{ref}, x_i) = {X_i}^c+{X^{ref}}^{\bar{c}}$. The average coalition target value is obtained by applying $f$ on the perturbated coalation dataset and averaging: $\hat{y}^c(x_i) = <f(\Chi(X, x_i)^c)>_+$. For each coalition, the gap between the coalition target value and the reference target value $\Delta^c(x_i) = \hat{y}^c(x_i)-\hat{y}^{ref}$ intuitively captures the impact of the coalition $c$. Therefore, the last step of the Kernel SHAP method consists of applying a weighted linear regression on $\Delta(x_i) = \{\Delta^c(x_i)\}_{c \in C}$ to  compute the approximated additive feature contributions. The closed form for the weighted regression is:

\begin{align}\label{eq:kernel_reg}
    \tilde{\phi}(x_i) = (W \cdot C^{T}C)^{-1}W \cdot C^{T} \Delta(x_i)
\end{align}

where $\tilde{\phi}$ is the estimated additive contributions of $f$ for the observation $x_i$ from Kernel SHAP method. As the coalitions $C$ are selected by order of weights $W$ in the Shapley values formula, the approximation $\tilde{\phi} \approx \phi$ is verified in practice if a sufficient number of coalitions is selected.

\section{Generalization to multiplicative contributions, X-SHAP}

\subsection{Theoretical extension: analytical solution to multiplicative contributions problem}

The X-SHAP algorithm adapts the Kernel SHAP method to multiplicative feature contributions. Thanks to the theoretical extension of the Shapley values, developed in~\citet{ortmann2013multiplicativeshap} in game theory, we easily extend the solution and desirable properties to the model-agnostic interpretabiltity problem.

In this section, we show that there is a unique solution of the multiplicative model-agnostic interpretability problem that verifies the geometrical efficiency (also refered to as local accuracy by~\citet{lundberg2017shapgithub}) and preserving-ratios properties.

\begin{property}\label{def:eff} \textbf{(Local accuracy)} Taking a predictive model $f$ associated with a dataset $(X, Y)$, the associated contributions function $\psi$ is geometrically efficient if it verifies the relation:
\begin{align}\label{eq:eff}
    \forall i \in [1, n],\  \psi^{0} \times \prod \limits_{j=1}^{m} \psi^{j}_{i}(x_i) = f(x_i) = \hat{y}_{i}
\end{align}
\end{property}

\begin{property}\label{def:ratios} \textbf{(Preserving-ratios)} For all $f$ and $(X,Y)$, the associated contributions $\psi$ is said to preserve ratios when one has:
\begin{align}\label{eq:ratios}
    \forall x \in X,
    \forall j_1 \ne j_2, \frac{\psi^{j_1}(x)}{\psi^{j_1}(c\setminus \{j_2\}, x)} = \frac{\psi^{j_2}(x)}{\psi^{j_2}(c\setminus \{j_1\}, x)}
\end{align}
\end{property}

\begin{theorem}\label{the:unique} For any predictive model $f$ associated with a dataset $(X, Y)$, there is a unique multiplicative feature contributions $\psi$ that is geometrically efficient and preserves ratios for the predictive model $f$ and for any observations $x \ X$. The solution is given by:
\begin{align}\label{eq:unique}
    \psi^{j}(x) = exp (\sum \limits_{c \subset F \setminus \{j\}} \frac{|c|!(|F| - |c| - 1)!}{|F|!}(\ln (f_{c \cup  \{j\}}(x_{c \cup \{j\}}))-\ln (f_{c}(x_{c}))))
\end{align}
\end{theorem}

\begin{definition}\label{def:ines} Given a predictive model $f$  and a dataset $(X, Y)$ and an observation $x$, a feature $j \in [1, m]$ is called inessential, if for every coalition $c \in F, j \notin c$, one has $f_{c \cup  \{j\}}(x_{c \cup \{j\}})= f_{c}(x_{c})$
\end{definition}

\begin{corollary}\label{cor:ines}  Given a predictive model $f$ associated with a dataset $(X, Y)$ and $j$ an inessential feature. Then, the contribution of the feature $j$, $\psi^j(x) = 1$.
\end{corollary}

\subsection{Practical extension: the X-SHAP algorithm}

Following the theoretical generalization of additive contributions to multiplicative contributions, X-SHAP extends the computation of the approximate multiplicative contributions $\tilde{\psi}(x_i)$ of each prediction $x_i \in X$: $\tilde{\psi}^{0} \times \prod \limits_{j = 1}^{m} \tilde{\psi}^{j}(x_i) = \hat{y}_{i}$. While facing the same computational challenges, Thus, the algorithm X-SHAP (Algorithm~\ref{alg:shapvalues_algo}) follows similar initial steps as the  SHAP, such as building a representative reference dataset $X^{ref}$ and selecting the coalitions $C$ with greatest weights. Then, as the predictive model is multiplicative, the whole algorithm of the Kernel SHAP has to be consequently adjusted. Thus, the arithmetic mean is transformed into geometric mean and the linear regression to a logarithm-generalized linear regression. The details of the algorithm are developed in Algorithm~\ref{alg:shapvalues_algo}.

\begin{algorithm}\label{alg:shapvalues_algo}
    \SetKwInOut{Input}{Input}
    \SetKwInOut{Output}{Output}

    function \textbf{x\_shap\_explainer} $(f, x_i, X^{ref}, C, W)$:
    
    \Input{$f$ the predictive function of the model, $x_i$ the observation to interpret, $X^{ref}$ the reference dataset, $C$ the $K$ selected coalitions, $W$ the associated weights of the coalitions}
    
    \Output{Vector $\tilde{\psi}_f(x_i)$ of X-SHAP contributions}
    
    $\hat{y}_{\times}^{ref} \leftarrow <f(X^{ref})>_{\times}$ \tcp{Average reference target value}  \
    $X_i\leftarrow x_i \times \mathds{1}_n$  \tcp{Augmented observation}\
    $\Chi^{c_k}(X, x_i) \leftarrow  X_i^{c_k}+X^{\bar{{c_k}}}, \forall\ {c_k}\ in\ C$ \tcp{Pertubated coalition datasets}\
    $\hat{y}_{\times}^c(x_i) \leftarrow <f(\Chi^c(X, x_i)>_{\times}, \forall\ {c_k}\ in\ C$ \tcp{Coalition average target values}\
    $\Delta_{\times}(x_i) \leftarrow (\hat{y}_{\times}^{{c}_1}(x_i)/\hat{y}_{\times}^{ref}, ...,\hat{y}_{\times}^{{c}_K}(x_i)/\hat{y}_{\times}^{ref})$  \tcp{Coalitions-reference gaps}\
    $C_{s} \leftarrow feature\_selection(C)$ \tcp{Feature selection using Lasso (optional)}  \
    $\tilde{\psi}_f(x_i) \leftarrow \exp((W \cdot C_{s}^{T}C_{s})^{-1}W\cdot C_{s}^{T}\ln(\Delta_{\times}(x_i)))$ \tcp{GLM to obtain contributions}
   
    \caption{X-SHAP for computing multiplicative feature contributions for a single observation $x_i$ following additive Kernel SHAP implementation~\citep{lundberg2017shapgithub}}
\end{algorithm}

Given a fixed number of selected coalitions, the complexity in time and space is polynomial.

\subsection{Interpretation}

\paragraph{Impact interpretation.} X-SHAP measures the multiplicative factor associated with a feature $j$ of the observation $x_i$. If the X-SHAP contribution $\psi^{j}(x_i) > 1$, the value of feature $j$ in observation $x_{i}$ increases the model prediction compared to the baseline. On the contrary, when $\psi^{j}(x_i) < 1$, the feature value decreases the model prediction from baseline. Finally, if $\psi^{j}(x_i) = 1$, the feature is inessential and thus impactless.

\paragraph{Link with log-GLMs.} In the specific case where the predictive model $f$ is a logarithmic Generalized Linear Model such as $\hat{y}_{i} = \exp(\alpha) \times \prod \limits_{j=1}^{k} \exp(\beta^{j} \times x_{i}^{j})$ where $\hat{y}_{i}$ is the prediction for observation $x_{i}$, $\beta^{j}$ is the coefficient for feature $j$ and $\alpha$ is a constant, the link between the multiplicative feature contributions $\psi^{j}$ and the coefficients $\beta^{j}$ of the GLM regression can be expressed as follows.

\begin{proposition} Let us assume features independence, then one has the following relation between terms of GLM's parameters $\beta^{j}$ and contributions $\psi^{j}(x_i)$:
\begin{align}\label{eq:xshap_val_defwithglm}
    \forall j \in [1, m], \psi^{j}(x_i) = \exp( \beta^{j} \times (x_{i}^{j} - <X^{j}>_+))
\end{align}
\end{proposition}

As expected, the multiplicative feature contribution measures the impact on the model output of the deviation of $x^{j}_{i}$ from expected value in $<X^{j}>_+$. Therefore X-SHAP allows a reconciliation with log-GLMs.

\section{X-SHAP metrics} 

In addition to the computation of multiplicative contributions, a set of tools is developed including metrics and visualizations. In this section, we present the main metrics used in section results.

\begin{definition} \textbf{X-SHAP multiplicative contributions of a group of observations.} The multiplicative contributions $\psi^{j}(G)$ of a group of distinct observations $G={\{x_i\}}_{i \subset [1,n]^{|G|}}$ is defined as the geometric mean of the multiplicative contributions of the observations $\psi^{j}(x_i)$ expressed as:
\begin{align}
    \psi^{j}(G) = <\psi_{j}(x_i)>_{\times, x_i \in G}
\end{align}
\end{definition}

\begin{definition} \textbf{X-SHAP local feature importance.} Let $I^{j}(x_i)$ denotes the local importance of feature $j$ for observation $x_i$. It measures the absolute multiplicative impact of the multiplicative contribution on the model's prediction. It is defined as:
\begin{align}
I^{j}(x_i) = \max(\frac{1}{\psi^{j}}(x_i), \psi^{j}(x_i))
\end{align}
\end{definition}

\begin{definition}\label{def:feature_importance} \textbf{X-SHAP global feature importance.} The global feature importance of the feature $j$, noted $I^{j}$, is defined as the geometric mean of local feature importances: \begin{align}
I^{j} = <I^{j}(x_i)>_{\times, i\in [1,n]}
\end{align}
\end{definition}

\begin{definition} \textbf{X-SHAP partial dependence.} Given a feature $j$ and a range of values $[x^j_1, x^j_2]$ of $x^j$, the partial dependence of the feature $j$ on $[x^j_1, x^j_2]$ is:
\begin{align}\label{eq:pd}
PD^j([x^j_1, x^j_2]) = \frac{<\psi^j([x^j_1, x^j_2])>_{\times}}{<\psi^j>_{\times}}\times <\hat{Y}>_{\times}
\end{align}
where $\psi^j([x^j_1, x^j_2])$ is the contribution vector of feature $j$ restricted to values $x^j_i \in [x^j_1, x^j_2], \forall i \in [1, n]$.
\end{definition}

\section{Data}

Three real-world datasets with continuous targets are used to present our results:

\begin{itemize}
    \item Boston dataset\footnote{\url{https://archive.ics.uci.edu/ml/machine-learning-databases/housing/}}: this dataset contains 13 numerical attributes and 506 observations. The regression task is to predict the median value of owner occupied houses. 
    \item Diabetes dataset\footnote{\url{https://www4.stat.ncsu.edu/~boos/var.select/diabetes.tab.txt}}: this dataset contains 10 numerical attributes and 442 observations. The regression task is to predict the progression of the disease one year after the baseline.
    \item Auto Insurance dataset\footnote{\url{https://www.kaggle.com/c/auto-insurance-fall-2017/data}}: this dataset contains 23 numerical and categorical attributes and 8161 observations. The task is to predict the severity of motor accidents as an expected material claim amount. 
\end{itemize}

Boston and Diabetes datasets are both sets for which the regression problem is easily solved. Moreover they both have a small number of features. These two characteristics make them good candidates to check the coherence and performance of the X-SHAP algorithm. 

The Auto Insurance dataset has more features. It is used to test the X-SHAP method on a real-world example when modeling experts (e.g. actuaries) would typically use GLMs in order to explore the multiplicative effects.

Each dataset is randomly split into a train set (70\% of original size) and a test set. Both a random forest regressor (RF) and a gradient boosting (GB) are fit on the training sets.

The reference data is taken from the training set and the X-SHAP values are computed on the test set.

\section{Results}

We analyze the results from different perspectives (local, global, and segmented) in order to verify the consistency between X-SHAP explanations, classical explanations tools and intuition.

\paragraph{Precision of approximations.}

First, we implement sanity checks to observe empirically properties satisfied by X-SHAP contributions:
\begin{enumerate}
    \item Local accuracy (property~\ref{def:eff}) is verified for predictions of the three datasets. The products of all the contributions are equal to the prediction with a mean percentage error $<10^{-16}$
    \item The estimation of the analytical multiplicative contributions (eq.~\ref{the:unique}) performed by the X-SHAP algorithm is accurate as soon as a sufficient number of coalitions is selected. We observe a quick convergence to analytical contributions. With the three datasets, the stability of the computations is reached when $|C|>500$.
\end{enumerate}

\begin{figure}[ht!]
    \centering
    
    \includegraphics[scale=0.25]{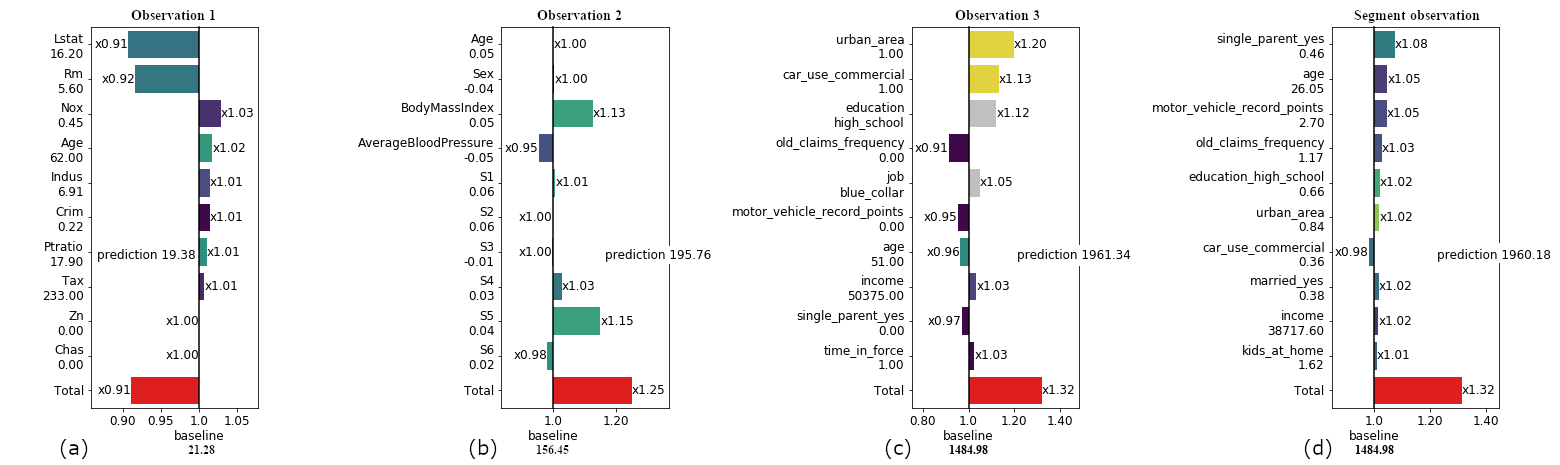}
    
    \caption{\textbf{X-SHAP multiplicative contributions.} X-SHAP multiplicative feature contributions $\tilde{\psi}^j(x_i)$ of top 10 features from the (a) Boston dataset, (b) Diabetes dataset and (c) Auto Insurance dataset. (d) Multiplicative contributions of young persons $\tilde{\psi}^j(G_{<30yo})$ from the Auto Insurance dataset. Read as follows: in (a), the prediction is $0.91$ times the baseline. $Lstat$ which the value is $16.20$ decreases the baseline by a factor of $0.91$ while the $age$ feature which the value is $62$ contributes by $\times 1.02$.}
    \label{fig:local_plot}
\end{figure}

\paragraph{Local explanations.}

Since X-SHAP provides a multiplicative breakdown of a model predictions, X-SHAP gives the possibility to locally depict, for each prediction $(x_{i}, \hat{y}_{i})$, how the values of the features contribute. In Figure~\ref{fig:local_plot}), starting from the reference value, the contributions are multiplied and have positive or negative impact on the final result (in red). These impacts depend on each observation value $x^j_i$.

\paragraph{Summary plots of contributions.}

We extend SHAP summary plots (\citep{lundberg2017shapgithub}) to analyze the impact of feature values to the model's prediction.

\begin{figure}[ht!]
    \centering
    \includegraphics[scale=0.225]{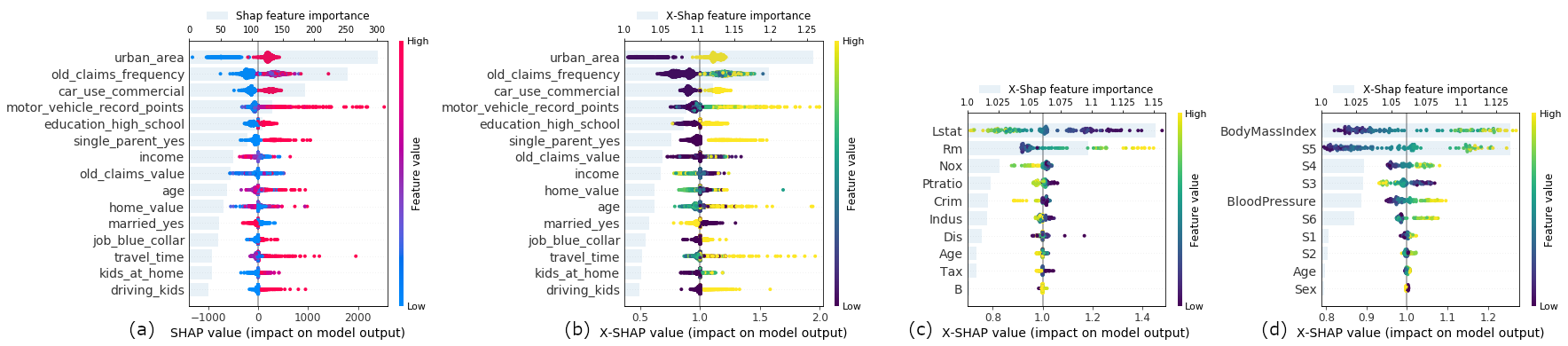}
    \caption{\textbf{Summary plots of contributions.} (a) and (b) : comparison of the Kernel SHAP and X-SHAP summary plots for top 15 features for all observations in $X_{test}$ of the Auto Insurance dataset and RF model. (a) Kernel SHAP additive values. (b) X-SHAP multiplicative values. (c) X-SHAP multiplicative values for the Boston data set and RF model. (d) X-SHAP multiplicative values for Diabetes dataset and RF model. 
    Dots represent pairs (contribution, feature). A heatmap associates the underlying feature value. Outliers are not displayed.
    The underlying bar chart represents the value of the global feature importance of each feature.}
    \label{fig:summary}
\end{figure}

Summary plots, depicted in Figure~\ref{fig:summary}, help to visualize how features interact with the model. Figure \ref{fig:summary}(a) presents the Kernel SHAP value ~\citep{lundberg2017shapgithub} while~\ref{fig:summary}(b) presents X-SHAP values. From these plots we can check consistency between the two algorithms. For most of the features presented there is a clear link between their value and their associated contribution, for example the feature $urban\_area$ identifies whether the person lives a in urban area (high density area). From the X-SHAP summary plot people living in dense areas have a higher average material claim cost than those living in rural areas. Similarly, people with a history of material claim cost ($old\_claims\_frequency$ feature) are more at risk to have material accidents.

\paragraph{Partial dependence of features.}

Estimating the overall marginal effect of a feature helps to better understand the relation between features and model output. Figure~\ref{fig:dependence} shows the comparison of the X-SHAP partial dependence $PD^j([x^j_i, x^j_{i+1}])$ with the partial dependence, defined in~\citet{hastie2009partialdependence}, for four different features from the Auto Insurance datasets. Both methodologies agree on the behavior of the dependency between the model and the features. Differences in values is mainly due to the way averages are computed: X-SHAP uses a geometric mean which is smaller than the arithmetic mean and less sensitive to outliers.

\begin{figure}[ht!]
    \centering
    \includegraphics[scale=0.23]{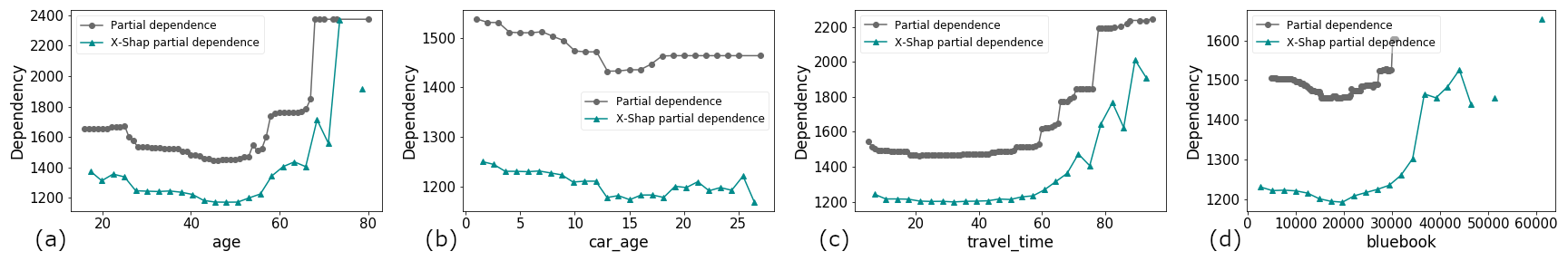}
    
    \caption{\textbf{Partial dependance plots.} Comparison between X-SHAP partial dependence $PD^j([x^j_1, x^j_2])$ (eq.~\ref{eq:pd}) and traditional additive partial dependence~\citep{hastie2009partialdependence} over four different features of the Auto Insurance dataset: (a) age, (b) car age, (c) travel time and (d) bluebook. For the X-SHAP dependence plots data was discretized in 25 bins.}
    \label{fig:dependence}
\end{figure}

\paragraph{Feature importance}
To understand a model from a global perspective, a used approach is the feature importance. Standard libraries implement such feature importance computation methods. X-SHAP feature importance is computed using the definition~\ref{def:feature_importance}. The larger the metric, the greater the effect of the feature on the model prediction.
Figure~\ref{fig:feature_importance} compares feature importance of RF model for Diabetes dataset: (a) inner implementation from RF model, (b) Kernel SHAP feature importance (defined as the mean of contribution absolute value), and (c) X-SHAP feature importance. Once again Kernel SHAP and X-SHAP assigns almost the same order of importance (only two order inversions). Moreover X-SHAP results are consistent with intuition since it is commonly acknowledged by experts that Body Mass Index is a major determinant of the evolution of the disease.

\begin{figure}[ht!]
    \centering
    \includegraphics[scale=0.27]{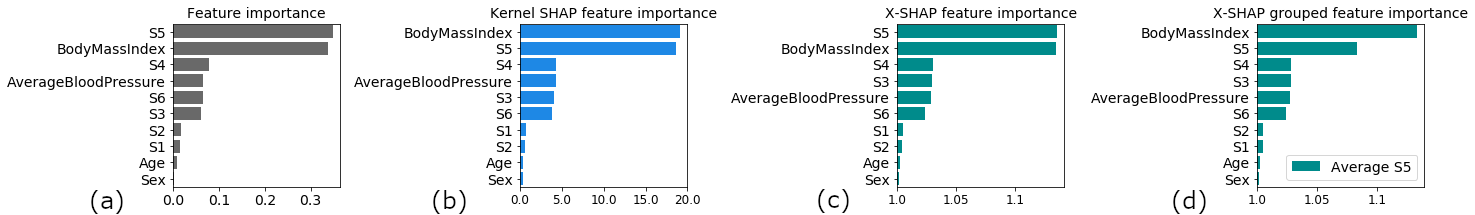}
    \caption{\textbf{Feature importance.} Comparison of the results of (a) feature importances given by the RF model, (b) Kernel SHAP feature importances, defined as the mean for each feature $j \in F$ of the absolute value of contribution for all observations $i$ and (c) X-SHAP feature importances $I^j$. (d) X-SHAP feature importances $I^j$ are depicted for the group of patients having a $S5$ feature value close to the average observed in the cohort.}
    \label{fig:feature_importance}
    
\end{figure}

\paragraph{Interpretation of a group of predictions.}

X-SHAP contributions can be aggregated to represent a certain group of observations sharing one or more characteristics, thus enabling another explanation level. This level can be adapted for all defined metrics: contributions, partial dependence and feature importances. For instance, Figure~\ref{fig:local_plot}(d) exhibits the interpretation of the young segment whereas Figure~\ref{fig:feature_importance}(d) presents the X-SHAP feature importance for the patients for which the $S5$ (lamotrigine blood measurement) feature value was close to the average observed in the cohort. While for the whole test set the features Body Mass Index and S5 have a similar effect magnitude, for this specific group there is a clear gap between the importance of these two features.

\section{Conclusion}

The increased need to providing highly accurate and interpretable multiplicative models has driven the development of X-SHAP, a model-agnostic interpreter that provides local approximations of the multiplicative contributions accompanied with theoretical proofs and empirical checks. In addition, we introduce the X-SHAP toolbox, a new set of tools to analyze local, global and segmented model structure by aggregating multiple local contributions of each or part of individual predictions.

Although the X-SHAP algorithm has a polynomial complexity, interesting opportunities regarding the decrease of complexity in time can arise while exploring the advantage of developing model-specific approximations of the multiplicative contributions for tree based ensemble models or neural networks.

\section*{Broader Impact}

X-SHAP offers a robust and model-agnostic methodology to assess multiplicative contributions. This unique method strengthens the set of techniques and tools contributing to making machine learning more transparent, auditable and accessible. This method is expected to prove useful for multiplicative underlying structures of modeled phenomena, such as areas where modelers are used to apply log-GLMs (e.g. actuaries modeling claims, epidemiology spreading modeling, disease risk factors estimation, energy consumption forecasting). It is provided as a tool that can help these experts adopt machine learning models with appropriate interpretability framework that stick to their habits.

\bibliographystyle{plainnat}
\bibliography{references}

\nocite{*}

\appendix

\section{Proofs}

\paragraph{Proof of Theorem 1.}

\begin{proof}
The proof of Theorem~\ref{the:unique} can be directly deducted from the results of Ortmann (2012)~\citep{ortmann2013multiplicativeshap} while traducing game theory problem to the interpretability problem. It originally derives direclty from the unicity of the additive shapley values. On one hand, the proof of the geometrical efficiency can be derived from additive version as follows:

\begin{align}
    \forall i \in [1, n],  \phi^{0} + \sum \limits_{j=1}^{k} \phi^{(j)}_{i} = \ln (\hat{y}_{i}) \\
    \nonumber \Leftrightarrow
    \forall i \in [1, n],  \ln (\psi^{0}) + \sum \limits_{j=1}^{k} \ln(\psi^{(j)}_{i}) = \ln (\hat{y}_{i}) \\
    \nonumber \Leftrightarrow
    \forall i,\  \exp (\ln (\psi^{(0)}) + \sum \limits_{j=1}^{k} \ln(\psi^{(j)}_{i}) )= \hat{y}_{i} \\
    \nonumber \Leftrightarrow
    \forall i,\  \psi^{(0)} \times \prod \limits_{j=1}^{k} \psi^{(j)}_{i} = \hat{y}_{i}
\end{align}

On the other hand, the proof of preserving-ratios is directly extended from the additive shapley values that preserve differences. Thus, given any coalition $c \in C$, it gives the following relations:

\begin{align}
    \forall j_1 \ne j_2, \phi^{j_1}(x) - \phi^{j_1}(c\setminus \{j_2\}, x) = \phi^{j_2}(x) -\phi^{j_2}(c\setminus \{j_1\}, x) \\
    \nonumber \Leftrightarrow
    \forall j_1 \ne j_2, \ln(\psi^{j_1}(x)) - \ln(\psi^{j_1}(c\setminus \{j_2\}, x)) = \ln(\psi^{j_2}(x)) - \ln(\psi^{j_2}(c\setminus \{j_1\}, x)) \\
    \nonumber \Leftrightarrow 
    \forall j_1 \ne j_2, \exp(\ln(\psi^{j_1}(x)) - \ln(\psi^{j_1}(c\setminus \{j_2\}, x))) = \exp(\ln(\psi^{j_2}(x)) - \ln(\psi^{j_2}(c\setminus \{j_1\}, x))) \\
    \nonumber \Leftrightarrow
    \forall j_1 \ne j_2, \frac{\exp(\ln(\psi^{j_1}(x)))}{\exp(\ln(\psi^{j_1}(c\setminus \{j_2\}, x)))} = \frac{\exp(\ln(\psi^{j_2}(x)))}{\exp(\ln(\psi^{j_2}(c\setminus \{j_1\}, x)))} \\
    \nonumber \Leftrightarrow
    \forall x \in X,
    \forall j_1 \ne j_2, \frac{\psi^{j_1}(x)}{\psi^{j_1}(c\setminus \{j_2\}, x)} = \frac{\psi^{j_2}(x)}{\psi^{j_2}(c\setminus \{j_1\}, x)}
\end{align}
All in all, as the additive Shapley values are the unique solution of the model-agnostic additive interpretable problem that respects local accuracy and preserve differences, the multiplicative Shapley values are the unique solution of the model-agnostic multiplicative interpretable problem that both respects local accuracy and preserve ratios.
\end{proof}

\paragraph{Proof of Proposition 1.}

\begin{lemma}\label{lemma:1} Given a predictive log-GLM model $f$ associated to a dataset $(X, Y)$ and $\psi$ the multiplicative shapley values, the relation between $\psi^0$ and $f$ is:
\begin{align}\label{eq:psi0}
\psi^{0} = \exp(\alpha)\prod \limits_{j=1}^{k} \exp(  \beta^{j} \times X^{j} ) 
\end{align}
\end{lemma}
\begin{proof}
Starting from the known relation in additive version between $\phi^0$ and $f$:
\begin{align}
    \nonumber
    \phi^{0} = <\hat{y}_{i}>_+ \\
    \nonumber \Leftrightarrow \psi^{0} = <\hat{y}_{i}>_{\times} \\
    \nonumber \Leftrightarrow
    \psi^{0} = (\prod \limits _{i=1} ^{n} \exp(\alpha) \times \prod \limits_{j=1}^{k} \exp(\beta^{j} \times x_{i}^{j})  )^{\frac{1}{n}} \\
    \nonumber \Leftrightarrow
    \psi^{0} = \exp(\alpha)\prod \limits_{j=1}^{k} \exp(  \beta^{j} \times \sum \limits _{i=1} ^{n} \frac{x_{i}^{j} }{n} ) \\
    \nonumber \Leftrightarrow
    \psi^{0} = \exp(\alpha)\prod \limits_{j=1}^{k} \exp(  \beta^{j} \times <X^j>_+ ) 
\end{align}
\end{proof}

\begin{lemma}\label{lemma:2} Given a predictive log-GLM model $f$ associated to a dataset $(X, Y)$ and $\psi$ the multiplicative shapley values, the relation between $\psi^j, \forall j \in [1,m]$ and $f$ is:
\begin{align}\label{eq:xshap_val_glm}
    \prod \limits_{j=1}^{m} \psi^{j}(x_i)  = \prod \limits_{j=1}^{m} \exp( \beta^{j} \times (x_{i}^{j} - <X^{j}>_+))
\end{align}
\end{lemma}

\begin{proof}
Introducing the expression of $\psi^{0}$ using the GLM's parameters found in eq.~\eqref{eq:psi0} into the two definitions of $\hat{y}_{i}$ (using log-GLM definition and feature contribution in eq.~\eqref{eq:xshap_def}) gives the following proof:
\begin{align}
    \nonumber
    \hat{y}_{i} = \psi^{0} \times \prod \limits_{j=1}^{k} \psi^{j}_{i} = \exp(\alpha) \times \prod \limits_{j=1}^{k} \exp(\beta^{j} \times x_{i}^{j} ) \\
    \nonumber \Leftrightarrow
    \exp(\alpha)\prod \limits_{j=1}^{m} \exp( \beta^{j} \times \bar{X}^{j} )  \times \prod \limits_{j=1}^{m} \psi^{j}(x_i)  = \exp(\alpha) \times \prod \limits_{j=1}^{m} \exp( \beta^{j} \times x_{i}^{j} )\\
    \nonumber \Leftrightarrow
    \prod \limits_{j=1}^{m} \psi^{j}(x_i)  = \prod \limits_{j=1}^{m} \exp( \beta^{j} \times (x_{i}^{j} - <X^{j}>_+))
\end{align}
\end{proof}

The proof of the \textbf{Proposition 1} is then straight forward.

\begin{proof}
From Lemma~\ref{lemma:2} and assuming the feature independence in eq.~\eqref{eq:xshap_val_glm}, the proof is straight forward:
\begin{align}
    \nonumber
    \prod \limits_{j=1}^{k} \psi^{j}(x_i)  = \prod \limits_{j=1}^{k} \exp( \beta^{j} \times (x_{i}^{j} - <X^{j}>_+))\\
    \nonumber \Rightarrow
    \forall j \in [1, m] \psi^{j}(x_i)  = \exp( \beta^{j} \times (x_{i}^{j} - <X^{j}>_+))
\end{align}
\end{proof}

\section{Sanity checks.}

We implement sanity checks to observe empirically properties satisfied by X-SHAP contributions.

\paragraph{Local accuracy.} Property~\ref{def:eff} is verified for predictions of the three datasets. The products of all the contributions are equal to the prediction with a mean percentage error $<10^{-16}$ (see Table~\ref{tab:model_xshap_score})

\paragraph{Precision of the approximations.} 
The X-SHAP algorithm performs an approximation of the analytical multiplicative contributions (eq.~\ref{the:unique}). The main approximation done is linked to the number of selected coalitions $|C|=n_{coalitions}$. In the algorithm X-SHAP, only the coalitions with biggest weights $W$ are considered. In Figure~\ref{fig:error_plot}, the value of the contributions are computed according to the number of considered coalitions. It appears that the convergence to stable contributions is relatively quick for the randomly selected observations. 

\begin{figure}[ht!]
    \centering
    
    \includegraphics[scale=0.3]{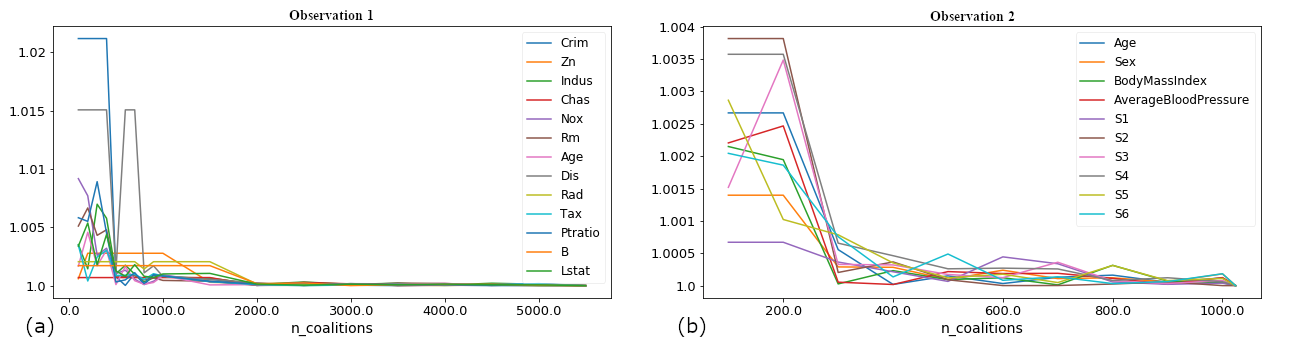}
    
    \caption{\textbf{Precision.} Relative error between X-SHAP multiplicative feature contribution value $\tilde{\psi}^j(x_i)$ found for a given number of coalitions $n\_coalitions$ and the X-SHAP multiplicative feature contributions found for the last value of $n\_coalitions$ displayed, for one observation from the (a) Boston dataset and (b) Diabetes dataset.}
    \label{fig:error_plot}
\end{figure}

\begin{table}
    \centering
    \begin{tabular}{||c|c||c|c|c|c|c|c|c||}
        \hline
        Data set & Model & MSE & R2 & mean\_APE & median\_APE & std\_APE & max\_APE\\
        \hline
        \multirow{2}{4em}{Boston} & RF & $4.96e-29$ & $1.0$ & $2.33e-16$ & $2.00e-16$ & $1.69e-16$ & $7.28e-16$ \\ 
        & GB & $3.44e-29$ & $1.0$ & $1.84e-16$ & $1.70e-16$ & $1.64e-16$ & $5.96e-16$ \\ 
        \hline
        \multirow{2}{4em}{Diabetes} & RF & $2.95e-27$ & $1.0$ & $2.79e-16$ & $2.40e-16$ & $2.15e-16$ & $9.54e-16$ \\ 
        & GB & $2.81e-27$ & $1.0$ & $2.61e-16$ & $2.14e-16$ & $1.90e-16$ & $8.80e-16$ \\ 
        \hline 
        \multirow{2}{4em}{Auto ED} & RF & $6.04e-25$ & $1.0$ & $3.10e-16$ & $2.52e-16$ & $2.41e-16$ & $1.5e-15$\\ 
        & GB & $6.83e-25$ & $1.0$ & $3.39e-16$ & $2.94e-16$ & $2.69e-16$ & $2.22e-15$ \\ 
        \hline 
    \end{tabular}
    
    \caption{\textbf{Local accuracy.} Scores of X-SHAP estimated contributions output against model predictions $\hat{y}$ for all three data sets and two models tested}
    \label{tab:model_xshap_score}
\end{table}

\end{document}